\documentclass[review]{elsarticle}

\usepackage{graphicx}
\usepackage{amsmath,amssymb} 
\usepackage{color}
\usepackage{amsthm}
\usepackage{amsthm}
\newtheorem{theorem}{Theorem}[section]

\theoremstyle{definition}

\theoremstyle{remark}

\journal{arXiv.org}

\begin{document}

\begin{frontmatter}

\title{Depth Image Upsampling based on Guided Filter with Low Gradient Minimization}


\author[mymainaddress]{Hang Yang\corref{mycorrespondingauthor}}
\cortext[mycorrespondingauthor]{Corresponding author}
\ead{yanghang@ciomp.ac.cn}

\author[mysecondaryaddress]{Zhongbo Zhang}

\address[mymainaddress]{Changchun Institute of Optics, Fine Mechanics and Physics, Chinese Academy of Science,
Changchun 130033, China}
\address[mysecondaryaddress]{Jilin University, Changchun, 130012,China}

\begin{abstract}
In this paper, we present a novel upsampling framework to  enhance the spatial resolution of the depth image.
In our framework, the upscaling of a low-resolution depth image is guided by a corresponding intensity images, we
formulate it as a cost aggregation problem with the guided filter. However, the guided filter does not make full use of the properties of the depth image. Since depth images have quite sparse gradients, it inspires us to regularize the gradients for improving depth upscaling results. Statistics show a special property of depth images, that is, there is a non-ignorable part of pixels whose horizontal or vertical derivatives are equal to $\pm 1$. Considering this special property, we propose a low gradient regularization method which reduces the penalty for horizontal or vertical derivative $\pm1$. The proposed low gradient regularization is integrated with the guided filter into the depth image upsampling method. Experimental results demonstrate the effectiveness of our proposed approach both qualitatively and quantitatively compared with the state-of-the-art methods.
\end{abstract}

\begin{keyword}
\texttt{Depth image}\sep upsampling \sep low gradient minimization \sep guided filter \sep regularization method
\end{keyword}

\end{frontmatter}


\section{Introduction}

Over the last decade, RGB-D sensors have made rapid development, such as  Microsoft Kinect, Intel Leap Motion and  ASUS Xtion Pro. They enable a variety of applications based on the depth image of the scenes, for instance, pose estimation \cite{Girshick} and scene understanding \cite{Gupta}. However, current depth cameras are limited by manufacturing and physical constraints. Hence, depth images are affected by degenerations due to noise, missing values, and typically have a low resolution \cite{Riegler}.
To mitigate the use of depth data, we need to recover the corresponding high-resolution (HR) depth image from a given low-resolution (LR) depth image.

Depth image upsampling is a quite challenging task. Specifically, due to the limited spatial resolution, the LR image loses or distorts fine structures in the HR image. A brute-force upscaling method often causes those structures which have sharp edges become blurred in the upsampled image. Especially for the case of single-image upscaling, the severely distorted fine structures often exists \cite{Hui}.

To address the above problem, a common approach is to utilize a corresponding HR intensity image as guidance \cite{Ferstl13,Kwon}. This is based on the fact that a correspondence between a depth edge and an intensity edge can be most likely established.
Some of the most successful algorithms for upsampling depth images aim at exploiting this correspondence assumption.

In this work, we present a novel method that combines the advantages of guided filter and the energy minimization model to compute an accurate high-resolution output from a LR depth map with the corresponding HR intensity image.
In recent years, the guided filter as a new edge-preserving technology has also been employed in a wide range of applications, such as image deconvolution \cite{YhangICIP}, image super-resolution \cite{He13} and image fusion \cite{Kang13}. Since the depth image is smooth, it is appropriate to be processed by guided filter who has shown to be effective for textureless image. And the guided filter can effectively fuse images from different sensors. Inspired by these, we attempt to use guided filter for depth image upsampling. However, the properties of depth images are not be fully exploited by the guided filter. A shallow observation that gradients of the depth image are 0 at most places. Therefore, together with the textureless property, we also regularize the depth maps with the sparse gradient prior in the meanwhile.

Based on the statistics of the depth image gradient, we are surprised to find that the sparse gradient model is not accurate enough \cite{Cai17}. In the depth image, although more than $80\%$ pixels have zero gradients (see Fig.1),
there is a non-ignorable part of pixels whose horizontal or vertical derivatives are equal to $\pm1$ (their proportion is about $15\%$, see Fig.1).
In other word, at many places, gradients of the depth map do not vanish but rather very small \cite{Cai17}.
This property has not been considered for depth image upsampling because it is not universal in natural images.
Hence, we develop a specific gradient regularization which is denoted  as $l^{t}_{0}$ gradient regularization.
Unlike the $l_{0}$ norm which penalizes the non-zero elements equally (the norm is always 1 if the element is not 0) \cite{Xu11,Nguyen15}, our proposed $l^{t}_{0}$ measure reduces the penalty for horizontal or vertical derivative 1
and thus allows for gradual depth changes.

The main contributions of this work are three-folds: (i) We present a specific gradient regularization $l^{t}_{0}$ which well describes the statistical property of the depth image gradients. (ii) We propose a solution to $l^{t}_{0}$  gradient minimization problem based on threshold shrinkage. (iii)We integrate the proposed $l^{t}_{0}$ gradient regularization with the guided filter into the the depth image upsampling method. In the experiments, we demonstrate that the proposed method provides competitive results compared with state-of the-art algorithms.

The rest of the paper is organized as follows: Section 2 briefly introduces related work in depth image upsampling.
Section 3 describes the proposed approach which considers the guided filter and the low gradient regularization.
In Section 4, we perform simulations on the benchmark dataset and show the effectiveness of our method.
We conclude the work in Section 5.

\section{Related Work}

There are many methods to perform depth image upsampling in the literature.
In general, they can be categorized into four classes:

\textbf{Exemplars based approaches.}These type of approaches build dictionaries for the LR and HR domains that are coupled by a common encoding. Yang $et~al.$ \cite{Yang10} seek the coefficients of this representation to obtain a upsampling result.
To improve the inference speed, Timofte $et~al.$ \cite{Timofte13} introduce the anchored neighborhood regression.
Li $et~al.$ \cite{Li12} present a joint examples-based upscaling approach. Ferstl $et~al.$ \cite{Ferstl15} present a dictionary learning method with edge priors for an anisotropic guidance. Schulter $et~al.$ \cite{Schulter15} use random regression forests instead of the flat code-book of sparse coding methods. Mahmoudi $et~al.$ \cite{Mahmoudi12} denoise noisy samples and learn a depth dictionary from noisy and denoised samples.

\textbf{Local image filtering.} Kopf $et~al.$ \cite{Kopf11} propose a joint bilateral filter based algorithm to smooth each depth pixel by considering the intensity similarity between the center pixel and its neighborhood. Yang $et~al.$ \cite{Yang07} present a method based on the bilateral filter that is iteratively used to generate an upsampled result. Geodesic distances is used to design the upsamapling weights in \cite{Liu13}, and Lu $et~al.$ \cite{Lu15} propose a smoothing approach to upscale depth map with the use of image segmentation. Li $et~al.$ \cite{LiDo16} develop a fast guided interpolation (FGI) approach based on weighted least squares, which densifies depth maps by global interpolation with alternating guidances.

\textbf{Global energy minimization methods.} These approaches formulate depth upscaling as an optimization problem
which employs data fidelity and regularization term. Diebel $et~al.$ \cite{Diebel06} develop Markov Random Field (MRF) based energy minimization framework, which fuses the LR depth map and the corresponding HR intensity image. In order to maintain local structures, Park $et~al.$ \cite{Park11} use a nonlocal means filter (MRF+NLM) to regularize the depth map. A more recent approach of
Ferstl $et~al.$ \cite{Ferstl13} utilizes an anisotropic diffusion tensor to guide the depth map upsampling (TGV), the tensor is calculated from a HR intensity image. Aodha $et~al.$ \cite{Aodha12} treat depth image upsampling as as MRF labeling problem which matches LR depth image patches to HR patches from an ancillary database. In \cite{Yang14}, an adaptive intensity guided regression method is proposed for depth upsampling.

\textbf{Deep learning based methods.} More recently, deep learning methods have become popular for single image upsampling.
A convolutional neural network (CNN) of three layers is trained in \cite{Dong14}, and Kim $et~al.$ \cite{Kim15} improve this approach substantially. Dong $et~al.$ \cite{Dong14PAMI} present an end-to-end upsampling convolutional neural network (SRCNN) to achieve image super-resolution. Xie $et~al.$ \cite{Xie16} propose a CNN framework for the single depth image upsampling guided by a reconstructed HR edge map. These methods have mainly been used to intensity images, where a great amount of training samples can be easily obtained. In contrast, huge data sets with dense, accurate depth maps have recently become available, e.g. \cite{Handa16}. Hui $et~al.$ \cite{Hui16} present a CNN framework in a multi-scale guidance architecture (MSG-Net).
Riegler $et~al.$ \cite{Riegler} integrate a energy minimization model with anisotropic TGV regularization into a end-to-end convolutional network for a single depth image upsampling.

\section{Depth Image Upsampling}

Given a original HR intensity image $I_{H}$ and a LR depth image $I_{L}$, we hope to obtain a HR depth map $u$.  If $I_{H}$ is a RGB image, it should be converted from the RGB space to the gray space. We first generate a coarse estimated depth image $D_{\uparrow}$ by bicubic interpolation from $I_{L}$, the resolution of $D_{\uparrow}$ and $I_{H}$ are equal.

In conventional image restoration problems, the guided filter performs very well in terms of both quality and efficiency \cite{YhangICIP}\cite{He13}. The filter can smooth image with the edge-preserving property as the bilateral filter, but there is no gradient reversal artifacts. The advantage of guided filter is very suitable for processing depth image, so we introduce  the guided filter into the upsampling algorithm. As discussed in Section 1, we first add the sparse gradient regularization for depth upsampling. Altogether we construct a new formula as follows:

\begin{equation}\label{eq01}
  \min_{u} \parallel u - D_{\uparrow} \parallel^{2}_{2} + \rho \parallel u - GF(u,I_{H}) \parallel^{2}_{2} + \eta \parallel \nabla u \parallel_{0}
\end{equation}
where $\nabla u =(u_{x},u_{y}) =(\partial_{x} \ast u,\partial_{y} \ast u)$ is the gradient of $u$, $\partial_{x} = [1,-1]$ and $\partial_{y} = [1,-1]^{T}$ are the horizontal and vertical derivative operator, respectively, $\rho$ and $\eta$ are two regularization parameters, we call this method \textbf{GFL0} .

Because $\textbf{GF}(\cdot,\cdot)$ is highly nonlinear, it is difficult to solve the problem directly. Following the solution in \cite{YhangICIP}, we employ a split variable approach
to solve the Eq.(\ref{eq01}) and the variables are iteratively updated:

\begin{eqnarray}
  z &=& GF(u,I_{H})\label{eqns1.2}\\
  u &=& \arg \min_{u} \parallel u - D_{\uparrow} \parallel^{2}_{2} + \rho \parallel u - z \parallel^{2}_{2} + \eta \parallel \nabla u \parallel_{0} \label{eqns1.1}
\end{eqnarray}
The minimization problem in Eq.(\ref{eqns1.1}) is widely used in image restoration models, many approaches have been proposed to solve it directly and approximately \cite{Xu11}\cite{Storath14}\cite{Nguyen15}.

$l_{0}$ gradient minimization does not always perform well \cite{Cai17}, one of the reasons may be the solution is only an approximation.
That is to say, we do not make full use of the properties of the gradient maps.  We count the gradient histogram of depth images and find out that the sparse gradient assumption is not accurate enough. In addition to 0, there is a non-ignorable part of pixels have horizontal or vertical derivative $\pm1$. In $l_{0}$ regularization, all nonzero gradients are penalized equally\cite{Cai17}.
Based on the special property of depth images, we propose a new gradient regularization algorithm to reduce the penalty for horizontal or vertical derivative $\pm1$.

\subsection{$l_{0}^{t}$ gradient minimization}
In this subsection, we will show the statistics of depth image gradients and describe the proposed $l_{0}^{t}$ gradient regularization method.

\begin{figure*}[!t]\label{fig.1}
\begin{minipage}[t]{0.5\linewidth}
\centering
\hspace*{-0.75cm}
\includegraphics[width=1.15\linewidth]{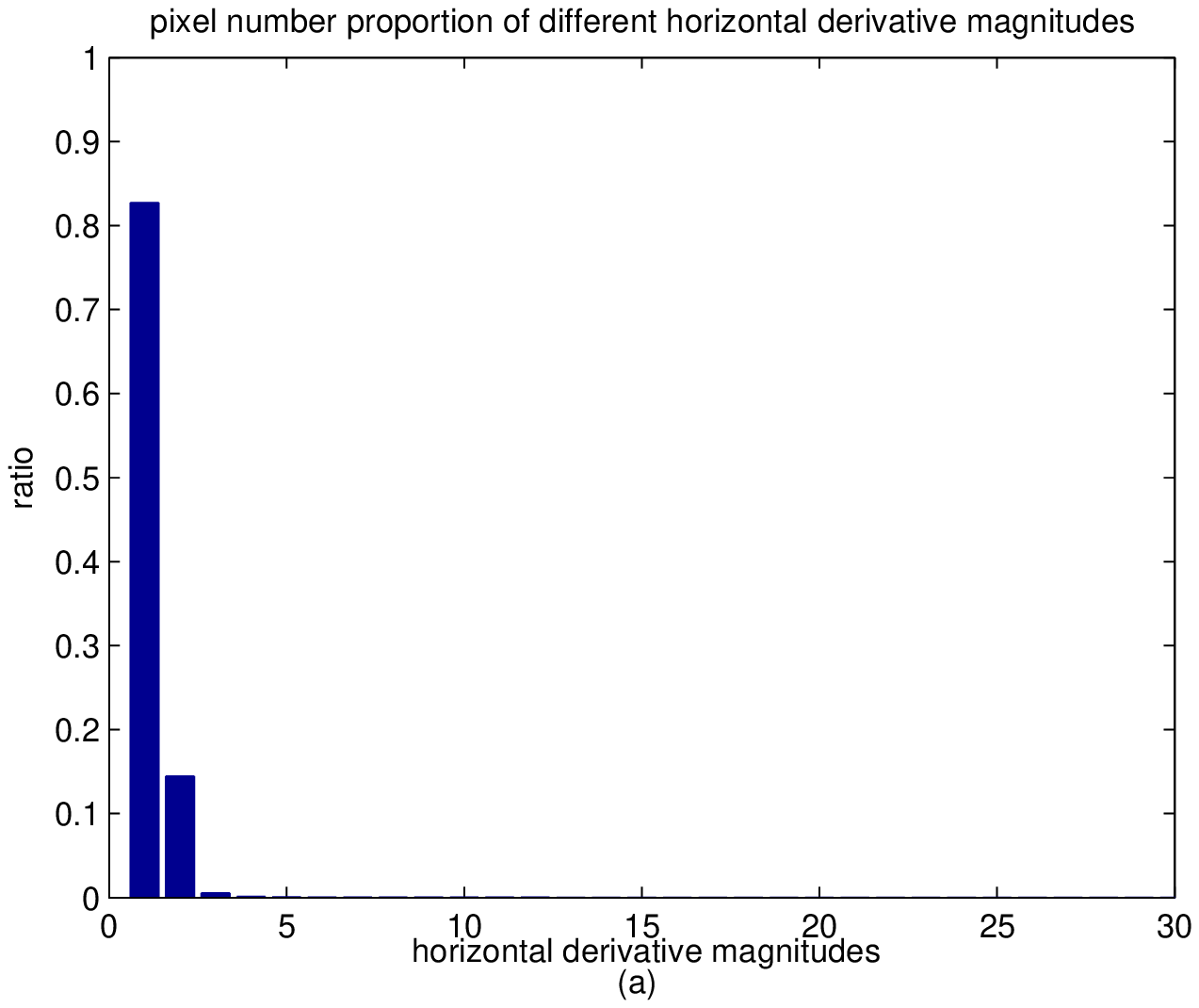}
\label{fig:side:a}
\end{minipage}%
\begin{minipage}[t]{0.5\linewidth}
\centering
\hspace*{-0.5cm}
\includegraphics[width=1.15\linewidth]{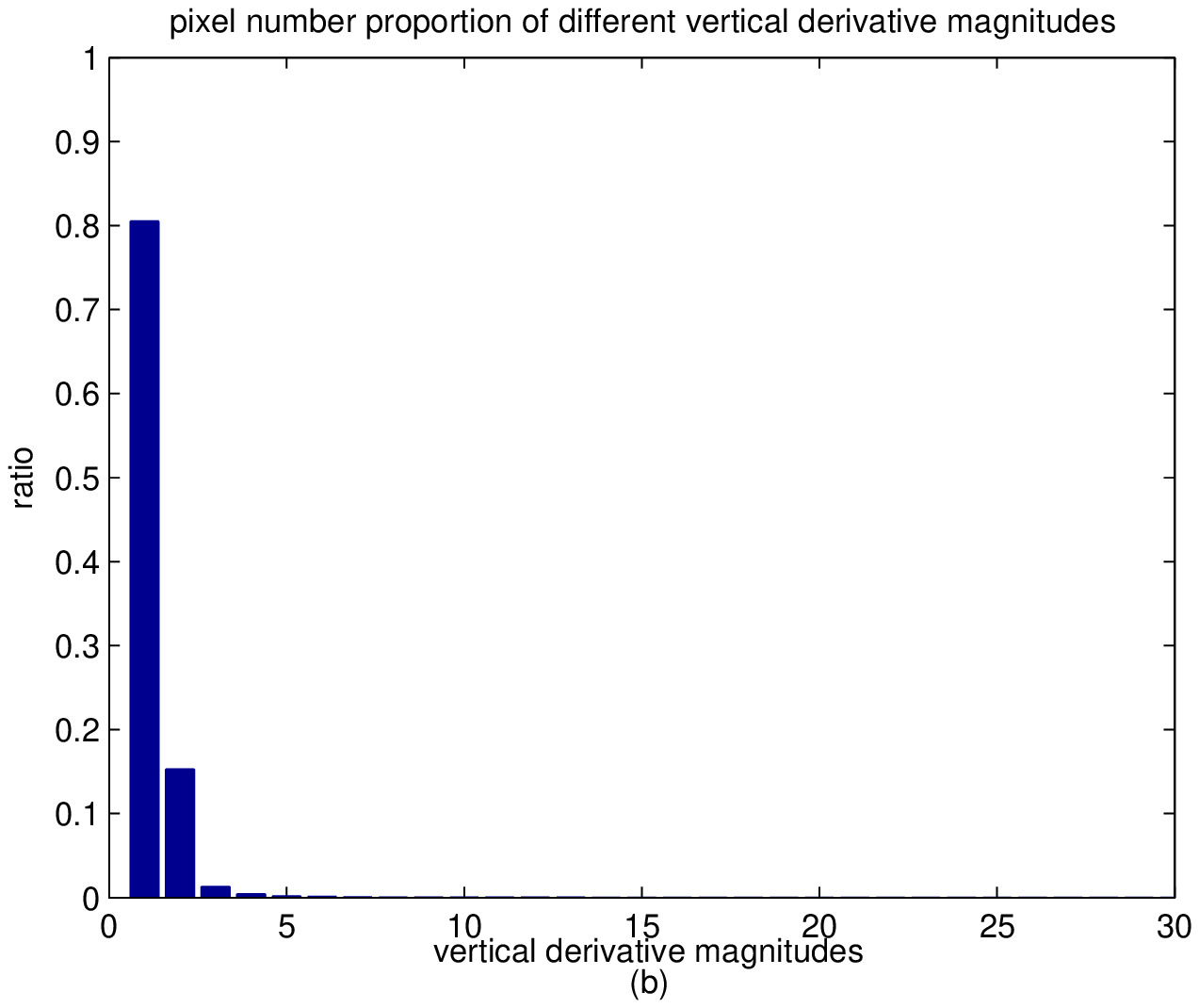}
\label{fig:side:b}
\end{minipage}
\vspace*{-0.75cm}
\caption{Horizontal and vertical derivative magnitude histograms of groundtruth disparity map of
Middlebury Stereo Datasets. a) horizontal derivative.(b) vertical derivative. We can see most pixels have gradient magnitude 0 and a non-ignorable part have magnitude 1.}\label{fig01}
\end{figure*}

In this work,  we use Middlebury Stereo Datasets(2001, 2003,2005,2006 and 2014) to do statistics on the depth gradients, and show the horizontal and vertical derivative magnitude histograms in Fig.\ref{fig01}. It is observed that depth image gradients cannot be simply described as sparse. We can see that most pixels have gradient magnitude 0 and a non-ignorable part whose gradients are $(\pm 1,\pm1)$, $(0,\pm1)$ or $(\pm1, 0)$.  Similar to the Total variation (TV) model, we propose a $l_{0}^{t}$ norm to reduce the penalty for horizontal and vertical derivative 1.

We define the  $\parallel \cdot \parallel_{l^{t}_{0}}$ norm as:
\begin{equation}\label{sec1.e02}
     \parallel \nabla u \parallel_{l^{t}_{0}} = \parallel u_{x} \parallel_{l^{t}_{0}} + \parallel u_{y} \parallel_{l^{t}_{0}}
\end{equation}
where
\begin{equation}\label{sec1.e03}
    \parallel f \parallel_{l^{t}_{0}} =\sharp\{i|\mid f_{i}\mid > 1\} + t\sharp\{i| \mid f_{i}\mid = 1\}
\end{equation}
and $0<t<1$, $\sharp\{\cdot\}$ is the number of elements in the data. In order to construct a problem that can be applied to a continuous domain, we revise the definition of $\parallel \cdot \parallel_{l^{t}_{0}}$ as
\begin{equation}\label{sec1.e031}
    \parallel f \parallel_{l^{t}_{0}} =\sharp\{i| \mid f_{i}\mid > 1\} + t\sharp\{i| 0 < \mid f_{i}\mid \leq 1\}
\end{equation}

Actually the $l_{0}^{t}$ ``norm'' is not a proper norm because it is not homogeneous, therefore we call it a measure.
Based on the statistics on the depth gradients, we set $t = 0.75$ in all the experiments.

Thus we use the $\parallel \cdot \parallel_{l^{t}_{0}}$ measure in place of $l_{0}$ norm in Eq.(\ref{eq01}), it leads to the following optimization model:

\begin{equation}\label{sec1.e04}
  \min_{u} \parallel u - D_{\uparrow} \parallel^{2}_{2} + \rho \parallel u - GF(u,ref) \parallel^{2}_{2} + \eta \parallel \nabla u \parallel_{l^{t}_{0}}
\end{equation}

Similar to Eq.(\ref{eq01}), we extend the split variable method to solve Eq.(\ref{sec1.e04}). In this case, two subproblems are as follows:

\begin{eqnarray}
  z &=& GF(u,I_{H})\label{eqns2.2}\\
  u &=& \arg \min_{u} \parallel u - D_{\uparrow} \parallel^{2}_{2} + \rho \parallel u - z \parallel^{2}_{2} + \eta \parallel \nabla u \parallel^{t}_{l_{0}}\label{eqns2.1}
\end{eqnarray}

To solve Eq.(\ref{eqns2.1}),  we introduce auxiliary variables $h$ and $v$, corresponding to $u_{x}$
and $u_{y}$ respectively, and rewrite the cost function as:

\begin{eqnarray}
\nonumber
  \{u,h,v\} &=& \arg \min_{u,h,v} \parallel u - D_{\uparrow} \parallel^{2}_{2} + \rho \parallel u - z\parallel^{2}_{2} \\
    &+& \beta (\parallel h - u_{x}\parallel^{2}_{2} + \parallel v  - u_{y}\parallel^{2}_{2}) + \gamma (\parallel h \parallel^{t}_{l_{0}} + \parallel  v \parallel^{t}_{l_{0}})\label{eqns3.1}
\end{eqnarray}
where $\beta$ is an automatically adapting parameter. Eq.(\ref{eqns3.1}) can be solved through alternatively minimizing $(h,v)$ and $u$, and it is broken into three subproblems in this work:

\begin{eqnarray}
 \nonumber
  u &=& \arg \min_{u} \parallel u - D_{\uparrow} \parallel^{2}_{2} + \rho \parallel u - z\parallel^{2}_{2}\\
    &~&~~~~~~~~~~+ \beta (\parallel u_{x} - h\parallel^{2}_{2} + \parallel u_{y} - v \parallel^{2}_{2})\label{eqns3.1}\\
  h &=& \arg \min_{h} \parallel h - u_{x}\parallel^{2}_{2} +  \lambda \parallel  h \parallel^{t}_{l_{0}} \label{eqns3.2}\\
  v &=& \arg \min_{v} \parallel v - u_{y}\parallel^{2}_{2} +  \lambda \parallel  v \parallel^{t}_{l_{0}} \label{eqns3.3}
\end{eqnarray}

The Eq.(\ref{eqns3.1}) is quadratic and thus has a global minimum.
We use Fast Fourier Transform (FFT) to speedup the diagonalization of derivative operators. These yield solutions in the Fourier domain
\begin{eqnarray}
  \mathcal{F}(u) &=& \frac{\mathcal{F}(D_{\uparrow})+ \rho \mathcal{F}(z) + \beta(\mathcal{F}(\partial_{x})^{*}\mathcal{F}(h) + \mathcal{F}(\partial_{y})^{*}\mathcal{F}(v))}{1+ \rho + \beta(|\mathcal{F}(\partial_{x})|^{2} +|\mathcal{F}(\partial_{y})|^{2})}
\end{eqnarray}
where $\mathcal{F}$ and $\mathcal{F}(\cdot)^{*}$ denote the FFT operator and the complex conjugate, respectively. The plus, multiplication, and division are all component-wise operators.

Now, the remaining question is how to solve the Eqs.(\ref{eqns3.2}) and (\ref{eqns3.3}).
In next subsection, we will show that these two apparently sophisticated subproblems have closed-form solutions and can be solved quickly.

\subsection{The $l_{0}^{t}$ measure minimization}

Without loss of generality, Eqs.(\ref{eqns3.2}) and (\ref{eqns3.3}) can be written in a unified way:
\begin{equation}\label{sec2.e01}
    \sum_{i} \min_{x_{i}} \{(x_{i} - p_{i})^{2} + \alpha H^{t}(p_{i})\}
\end{equation}
where
\begin{equation}\label{sec2.e02}
H^{t}(p)=
   \begin{cases}
   0 &\mbox{if p = 0}\\
   t &\mbox{if $0 <\mid p \mid \leq 1$}\\
   1 &\mbox{if $\mid p \mid > 1$}
   \end{cases}
\end{equation}
Each single term w.r.t. pixel $p_{i}$ in Eq. (\ref{sec2.e02}) is
\begin{equation}\label{t.1}
    E(p) = \min_{p} (x - p)^{2} + \alpha H^{t}(p)
\end{equation}
For Eq.(\ref{sec2.e01}), we can obtain its closed-form solution based on the following theorem.
In this work, we discuss two cases, that is, the situation of $|x|\geq 1$ and the situation of $|x| < 1$. Theorem 3.1 and Theorem 3.2 give the solution of the Eq.(\ref{sec2.e01}) in the case of $|x|\geq 1$, and Theorem 3.3 shows the solution of the Eq.(\ref{sec2.e01}) in the case of $|x| < 1$.

\begin{theorem}

When $|x|\geq 1$, Eq.(\ref{t.1}) reaches its minimum $E^{*}_{p}$ under the condition
\begin{equation}\label{t.12}
   p =
   \begin{cases}
   0 &\mbox{if $ \mid x \mid \leq \min(\frac{1+\alpha t}{2}, \sqrt{\alpha})$}\\
   $$sgn(x)$$ &\mbox{if $\frac{1+\alpha t}{2} < \mid x \mid \leq 1 + \sqrt{\alpha(1-t)}$}\\
   x &\mbox{if $ \mid x \mid > \max(1 + \sqrt{\alpha(1-t)}, \sqrt{\alpha})$}
   \end{cases}
\end{equation}
\end{theorem}\label{thm0b}

\begin{proof}
 Without loss of generality, we suppose $x\geq0$, and the proof of $x<0$ case is similar.

\textbf{1)} When $ x > \max(1 + \sqrt{\alpha(1-t)}, \sqrt{\alpha}) >1 $, non-zero $p > 1$ yields

\begin{equation}\label{lem.1}
    E_{p} = (x-p)^{2} + \alpha
\end{equation}
when $p =x$, Eq.(\ref{lem.1}) achieves minimum value $\alpha$.

Note that $p=0$ leads to
\begin{equation}\label{lem.2}
    E_{p} = x^{2} > \alpha
\end{equation}
And we can find that, $0<p\leq 1$ yields
\begin{equation}\label{lem.3}
    E_{p} = (x-p)^{2} + \alpha t
\end{equation}
when $p = 1$, Eq.(\ref{lem.1}) achieves minimum value $(x-1)^{2}+\alpha t$.

Because $ x > 1 + \sqrt{\alpha(1-t)}$, that is to say $ (x-1)^{2}>\alpha (1-t)$,
and $ (x-1)^{2} + \alpha t>\alpha$.

Comparing Eq.(\ref{lem.1}) and (\ref{lem.2}), the minimum energy $E_{p}$ is produced when $p = x$.

\textbf{2)} When $\frac{1+\alpha t}{2} <  x  \leq 1 + \sqrt{\alpha(1-t)}$,  $p > 1$ yields
\begin{equation}\label{lem.4}
    E_{p} = (x-p)^{2} + \alpha \geq \alpha \geq (x-1)^{2} +\alpha t
\end{equation}

$p=0$ leads to
\begin{equation}\label{lem.5}
    E_{p} = x^{2} > (x-1)^{2} +\alpha t
\end{equation}

When $0<p\leq 1$, Eq.(\ref{lem.3}) still holds.
we find that, when $x \geq 1$, the minimum energy Eq.(\ref{lem.3}) is produced when $p = 1$, the minimum value is
$(x-1)^{2}+\alpha t$.

So, in this case, comparing these three values, when $ p = sgn(x) = 1$, $E_{p}$ achieves minimum.

\textbf{3)} When $ x \leq \min(\frac{1+\alpha t}{2}, \sqrt{\alpha})$, Eq.(\ref{lem.1}) still holds, the minimum energy $\alpha$ is greater
than $ x^{2}$.

When $p=0$, $E_{p}$ has its minimum value $x^{2}$.

When $0<p\leq 1$, if $x \geq 1$, Eq.(\ref{lem.3}) can achieve  minimum value $(x-1)^{2}+\alpha t \geq x^{2}$.

So, the minimum energy $E_{p}$ is produced when $p = 0$
\end{proof}

\begin{figure*}[!t]
\begin{minipage}[t]{0.5\linewidth}
\centering
\hspace*{-0.75cm}
\includegraphics[width=1.1\linewidth]{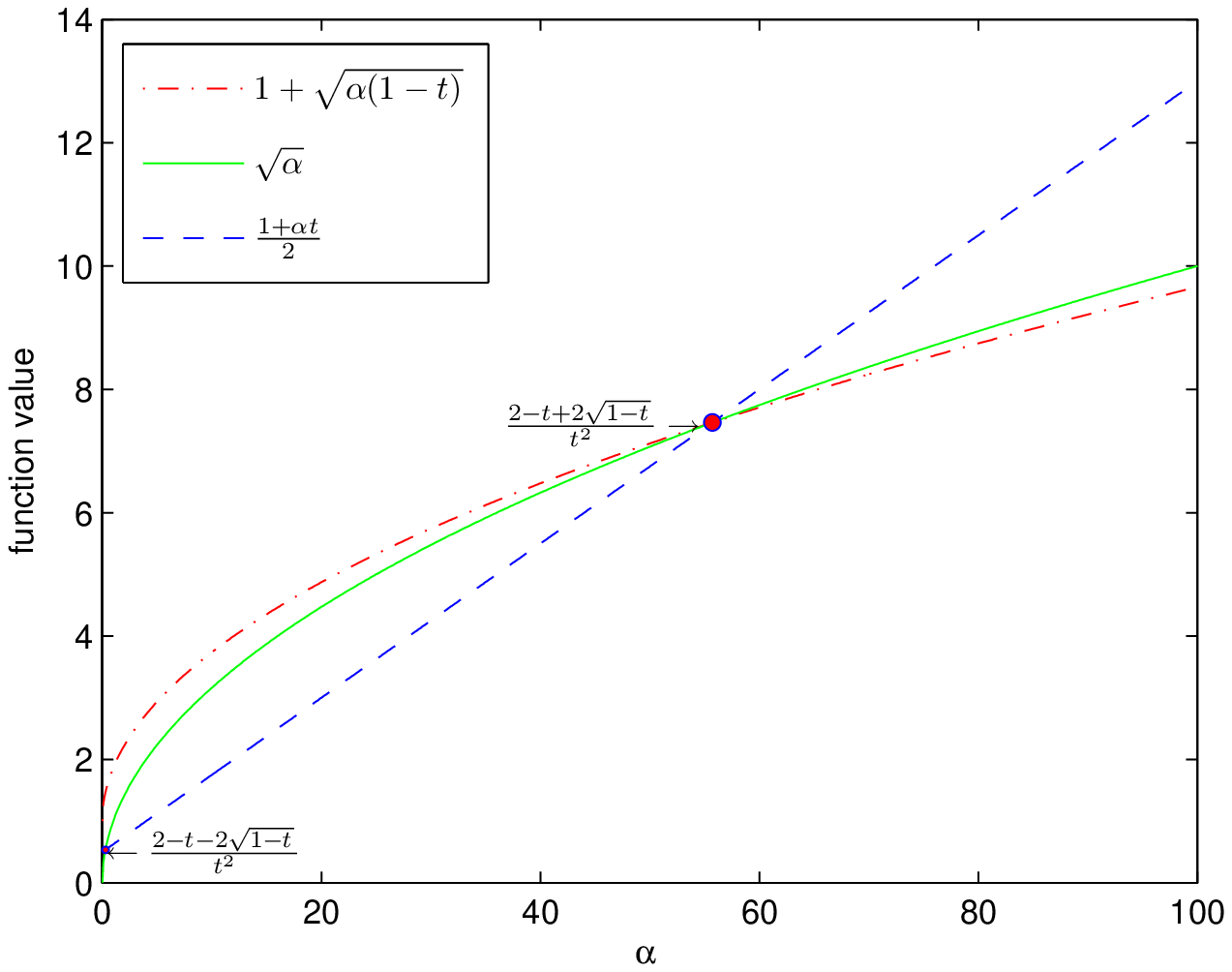}
\label{fig:side:a}
\end{minipage}%
\begin{minipage}[t]{0.5\linewidth}
\centering
\hspace*{0.0cm}
\includegraphics[width=1.1\linewidth]{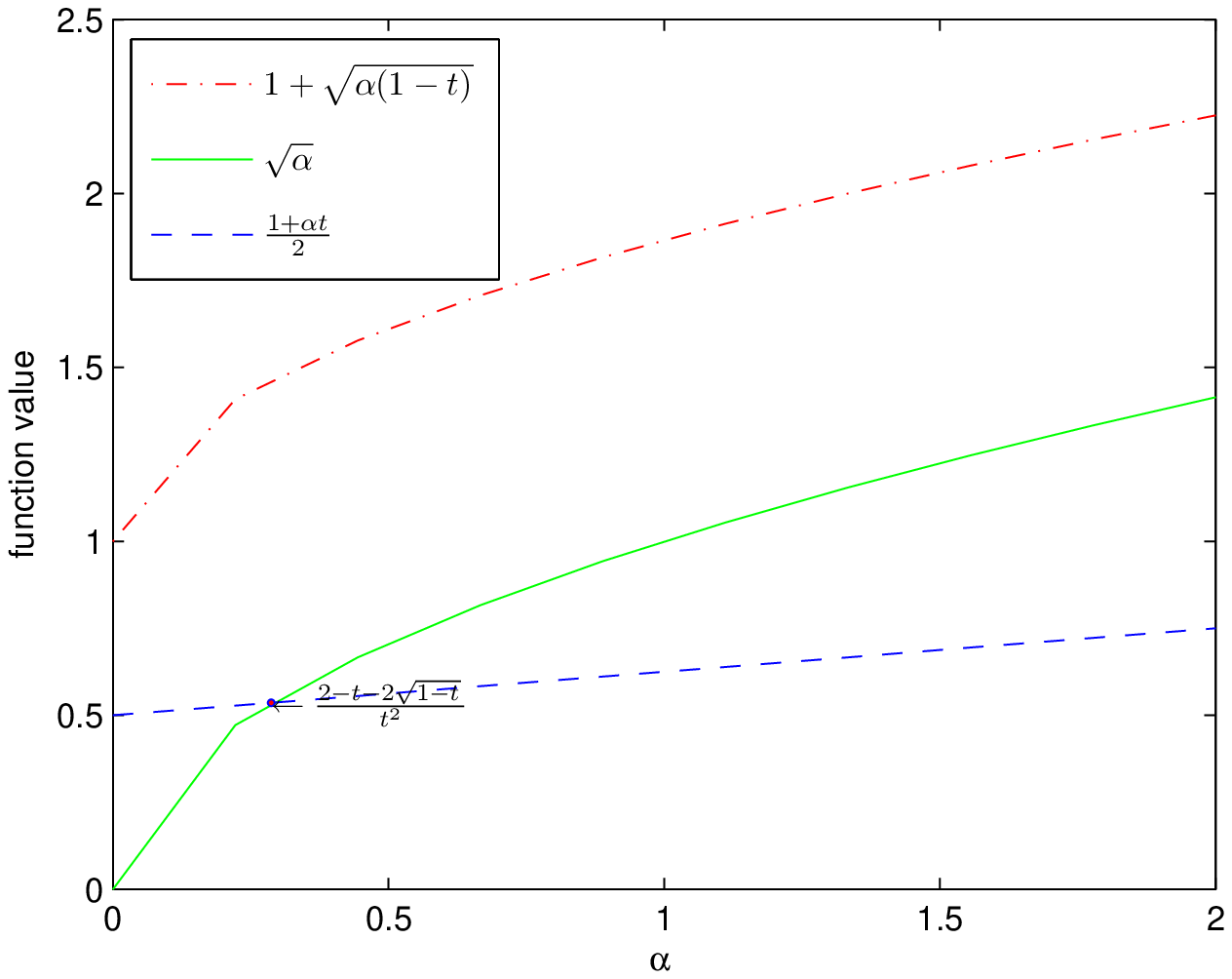}
\label{fig:side:b}
\end{minipage}
\vspace*{-0.35cm}
\caption{The relation of the three curves when $t = 0.25$. Left subplot: $\alpha \in [0,100]$. Right subplot: $\alpha \in [0,2]$.}\label{figCurve}
\end{figure*}

In Theorem 3.1, the relationship of three functions $1 + \sqrt{\alpha(1-t)}$, $\sqrt{\alpha}$ and $\frac{1+\alpha t}{2}$ is not determined, and therefore the reliability of the conclusion can not be guaranteed.
Through further research, we show the three function curves in Fig.\ref{figCurve}, and one can clearly see the relationship between them (Strict mathematical proofs of the relationship can be found in the supplementary material).

Now, according to the above analysis, we can rewrite Theorem 3.1 as follows:

\begin{theorem}

When $|x| \geq 1$, and $\alpha > \frac{2-t + 2\sqrt{1-t}}{t^{2}}$, Eq.(\ref{t.1}) reaches its minimum $E^{*}_{p}$ under the condition
\begin{equation}\label{t.3.1}
   p =
   \begin{cases}
   0 &\mbox{if $ \mid x \mid \leq  \sqrt{\alpha}$}\\
   x &\mbox{if $ \mid x \mid > \sqrt{\alpha}$}
   \end{cases}
\end{equation}

When $|x|\geq 1$, and $\frac{2-t + 2\sqrt{1-t}}{t^{2}} \geq \alpha \geq \frac{2-t - 2\sqrt{1-t}}{t^{2}}$, Eq.(\ref{t.1}) reaches its minimum $E^{*}_{p}$ under the condition
\begin{equation}\label{t.3.2}
   p =
   \begin{cases}
   0 &\mbox{if $ \mid x \mid \leq \frac{1+\alpha t}{2}$}\\
   $$sgn(x)$$ &\mbox{if $\frac{1+\alpha t}{2} < \mid x \mid \leq 1 + \sqrt{\alpha(1-t)}$}\\
   x &\mbox{if $ \mid x \mid > 1 + \sqrt{\alpha(1-t)}$}
   \end{cases}
\end{equation}

When $|x| \geq 1$, and $\frac{2-t - 2\sqrt{1-t}}{t^{2}} > \alpha > 0$, Eq.(\ref{t.1}) reaches its minimum $E^{*}_{p}$ under the condition
\begin{equation}\label{t.3.3}
   p =
   \begin{cases}
   $$sgn(x)$$ &\mbox{if $1 \leq \mid x \mid \leq 1 + \sqrt{\alpha(1-t)}$}\\
   x &\mbox{if $ \mid x \mid > 1 + \sqrt{\alpha(1-t)}$}
   \end{cases}
\end{equation}
\end{theorem}\label{thm01}

The details proofs of the Theorem 3.2 can be found in the supplementary material.

\begin{theorem}

When $|x| < 1$, Eq.(\ref{t.1}) reaches its minimum $E^{*}_{p}$ under the condition
\begin{equation}\label{t.12}
   p =
   \begin{cases}
   0 &\mbox{if $ \mid x \mid \leq \sqrt{\alpha t}$}\\
   x &\mbox{if $ \mid x \mid > \sqrt{\alpha t}$}
   \end{cases}
\end{equation}

\end{theorem}\label{thm02}
\begin{proof}
Without loss of generality, we suppose $x\geq0$, and the proof of $x<0$ case is similar.

\textbf{1)} when $ x^{2} \leq \alpha t $, non-zeros $0<p\leq 1$ yields
\begin{equation}\label{lem.6}
    E_{p} = (x-p)^{2} + \alpha t \geq \alpha t \geq x^{2}
\end{equation}

Note that $p=0$ leads to  $E_{p} = x^{2}$. Comparing Eq.(\ref{lem.6}), the minimum energy $E_{p}$ is produced when $p = 0$.

\textbf{2)} when $\alpha t < x^{2} < 1$,  $p = 0$ leads to $E_{p} = x^{2}$. But when $p=x$, $E_{p}$ has its minimum value $\alpha t$.
Comparing these two values, the minimum energy $E_{p}$ is produced when $ p = x$.
\end{proof}

Our depth image upsampling algorithm is sketched in \textbf{Algorithm 1}. Parameter
$\beta$ is automatically adapted in iterations starting from a small value $\beta_{0}$, it is multiplied by $\kappa$ each time.
This scheme is inspired by Wang $et~al$ \cite{Wang08}, which shows that this scheme is effective to speed up convergence.
In all the experiments, we fix the regularization parameters $\beta_{0} = 0.0025$ and set $\kappa = 2$ to balance between efficiency and performance.

---------------------------------------------------------

\textbf{Algorithm 1: Depth Image Upsampling Algorithm}

1. Input: LR depth image $I_{L}$, HR intensity image $I_{H}$, parameter $\beta_{0}$

~~~~initialize $u \leftarrow D_{\uparrow}$, $\beta = \frac{1}{2}\beta_{0}$, $\lambda = 255 \times \frac{\beta}{\beta_{0}}$, $\rho = 0.1 \times \beta$ and $\kappa = 2$.

2. Repeat $p=1:Max_{iter}$:

~~~~~~solve for $z^{p}$ using Eq.(\ref{eqns2.2});

~~~~~~solve for $u^{p}$ using Eq.(14);

~~~~~~solve for $h^{p}$ and $v^{p}$ base on Theorem 3.2 and Theorem 3.3;

~~~~~~update parameters:  $\beta \leftarrow \kappa\beta$, $\lambda \leftarrow 255 \times \frac{\beta}{\beta_{0}}$ and $\rho \leftarrow 0.1 \times \beta$.

3. Output: HR depth image $u$.

---------------------------------------------------------

\section{Experimental results}

To verify the superiority of our method, we evaluate
the performance of the proposed method with respect to some state-of-the-art depth image upsampling methods. We perform experiments on a PC with Intel i7-5600U CPU (2.6 GHz) and 8 GB
RAM using MATLAB 2012b.  $20-30$ iterations are generally performed in our method. Most
computation is spent on FFT in Eq.(14) and guided image filtering  in Eq.(\ref{eqns2.2}).
Overall, it takes per iteration of 1.4 seconds to upsample $\times 4$ to $345 \times 272$ depth images.

In this experiment, we evaluate our method on two standard benchmark datasets for depth
map super-resolution: Following \cite{Park11,Ferstl13,Gupta}, we evaluate our results on the noisy Middlebury 2007 dataset. Additionally, in a second evaluation, we compare our method on the challenging ToFMark dataset $Books$, $Devil$ and $Shark$ which as proposed in \cite{Ferstl13}.

\subsection{Noisy Middlebury}
In this experiment, we evaluate the performance of the proposed method on $Art$, $Books$, $Moebius$, $Dolls$, $Laundry$ and $Reindeer$  of the Middlebury dataset.
To simulate the acquisition process of a Time-of-Flight sensor, these input images are added depth dependent Gaussian noise,
as proposed by Park $et~al$ \cite{Park11}.

We compare our method to simple upsampling methods, such as bicubic interpolation. We compare our proposed method to other
approaches that utilize an additional intensity image as guidance. Those methods
include the Markov Random Field (MRF) based approach in \cite{Diebel06}, the joint bilateral filtering with cost volume (JBFcv) in \cite{Yang07}, cross-based local multipoint filtering (CLMF) in  \cite{Lu12}
the guided image filter (GIF) in \cite{He13}, the non-local means filter (MRF+NLM) in \cite{Park11}, the variational model (TGV) in \cite{Ferstl13} and fast guided global interpolation (FGI) in \cite{LiDo16}. In addition, to illustrate the effectiveness of the low gradient regularization, we also compare with GFL0 proposed in Eq.(1).

Table 1 reports quantitative results in terms of the Root Mean Squared Error(RMSE)
between ground truth depth maps and the results by various depth upsampling methods including ours.

From the quantitative results in Table 1,  we observe that the proposed method clearly
performs better than state-of-the-art methods that utilize an additional
guidance input for most images and upsampling factors.

The proposed method clearly outperforms
several existing methods like GF \cite{He13}, MRF+NLM \cite{Park11}, CLMF \cite{Lu12}
and TGV \cite{Ferstl13} that used different color-guided upsampling or optimization techniques.
Our method also yields much smaller error rates than FGI method \cite{LiDo16}.
Noting that the GFL0 approach performs generally better than TGV and achieves relatively close results to the proposed method.
Our approach always achieves the best results in RMSE
because it allows for gradual pixel value variation which is
common in depth images.

    \begin{table}[h]
    \caption{Error as root mean squared error(RMSE) comparison on middlebury 2007 datasets with added noise for magnification factors ($\times 2$ and $\times 4$). We highlight the best result in boldface}
    \centering
\resizebox{\textwidth}{20mm}{
    \begin{tabular}{ccccccccccccc}
    \hline
    & \multicolumn{2}{c}{Art} & \multicolumn{2}{c}{Books} & \multicolumn{2}{c}{Moebius} & \multicolumn{2}{c}{Dolls} & \multicolumn{2}{c}{Laundry} & \multicolumn{2}{c}{Reindeer}\\
    \cline{1-13}
    &$\times 2$ & $\times 4$ & $\times 2$ & $\times 4$&$\times 2$ & $\times 4$&$\times 2$ & $\times 4$ & $\times 2$ & $\times 4$&$\times 2$ & $\times 4$ \\
    \hline
    Bicubic& 4.78  & 5.54   &  4.20 & 4.38 & 4.16& 4.31 & 4.16  & 4.30 & 4.37  &  4.74 & 4.51 & 4.95 \\

    MRF\cite{Diebel06} & 3.49  & 4.51 &  2.06 & 3.00 & 2.31& 3.11 & - & - & - & - & - & -\\

    JBFcv\cite{Yang07}& 3.01  & 4.02   & 1.87 & 2.23 &1.92 &2.42 & - & - & - & - & - & -\\

    CLMF\cite{Lu12} & 3.29  & 4.03   & 1.80 & 2.38 &1.79 &2.29 & 1.83  & 2.37 & 2.36  &  2.91 & 2.52 & 3.15\\

    GIF\cite{He13}& 3.55  & 4.41  &  2.37 & 2.74 &2.48 &2.83 & 1.79  & 2.64 & 2.33  &  3.22 & 2.63 & 3.43 \\

    MRF+NLM  \cite{Park11}&  3.74 &4.56& 1.95 & 2.61  & 1.96 &2.51& 2.06  & 2.61 & 2.99  &  3.63 &3.11 & 3.86\\

    TGV\cite{Ferstl13}& 3.19  &  4.06   & 1.52&  2.21  &\textbf{1.47} &2.03& 1.49  &  2.85 & 2.62   &3.44& 2.78    & 3.20\\

    FGI\cite{LiDo16}& 3.13  & 4.14 &  1.48 &  1.92 &   1.65   & \textbf{1.94} & 1.47 & 1.84 & 1.94   &  \textbf{2.59} &2.22   &  2.86\\

    GFL0&  2.78 &3.98& 1.40 & 1.89 & 1.63 & 2.04& 1.42 &1.86 & 1.90 &  2.65 & 2.07  &2.79\\

    Ours & \textbf{2.71} &\textbf{3.87} &  \textbf{1.34} & \textbf{ 1.82}  &1.57 & 2.01& \textbf{1.38} &\textbf{1.79} & \textbf{1.83} &  2.60 & \textbf{ 2.01}  &\textbf{2.76}\\
    \hline
    \end{tabular}
}
    \end{table}\label{tabMidd01}

Fig.\ref{figart}--\ref{figreindeer} show the visualize qualitative results of the Middlebury data with zoomed cropped regions.
From the figures, we can notice that the proposed algorithm also generates more visual appealing results than
the state-of-the-art approaches.  Our algorithm can preserve thin structures of the scene in regions.
Edges in our results are generally smoother and sharper along the depth boundaries. Our algorithm also preserves thin structures
in regions. Although FGI\cite{LiDo16}, TGV\cite{Ferstl13} and GFL0 also generate good RMSE scores, the results of them suffer from artifacts around boundaries
visually.

\begin{figure*}[!t]
  \hspace{-1.5cm}
  \includegraphics[width=1.25\linewidth]{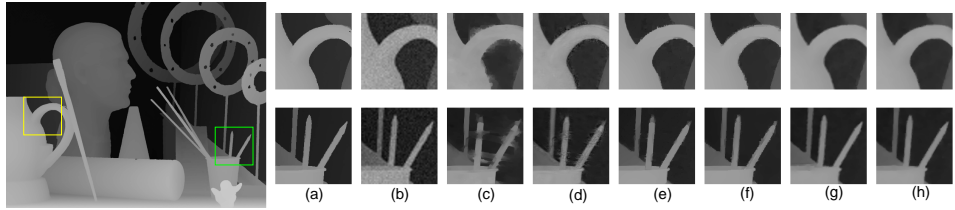}\\
  \vspace{-1cm}
  \caption{Visual comparison of $Art$ with cropped zoomed regions (scaling factor = 4).(a) Ground truth. (b) Bicubic. (c) MRF+NLM \cite{Park11}. (d) GIF\cite{He13}. (e) TGV\cite{Ferstl13}. (f) FGI\cite{LiDo16}. (g) GFL0. (h) Our proposed method}\label{figart}
\end{figure*}
\begin{figure*}[!t]
  \hspace{-1.5cm}
  \includegraphics[width=1.25\linewidth]{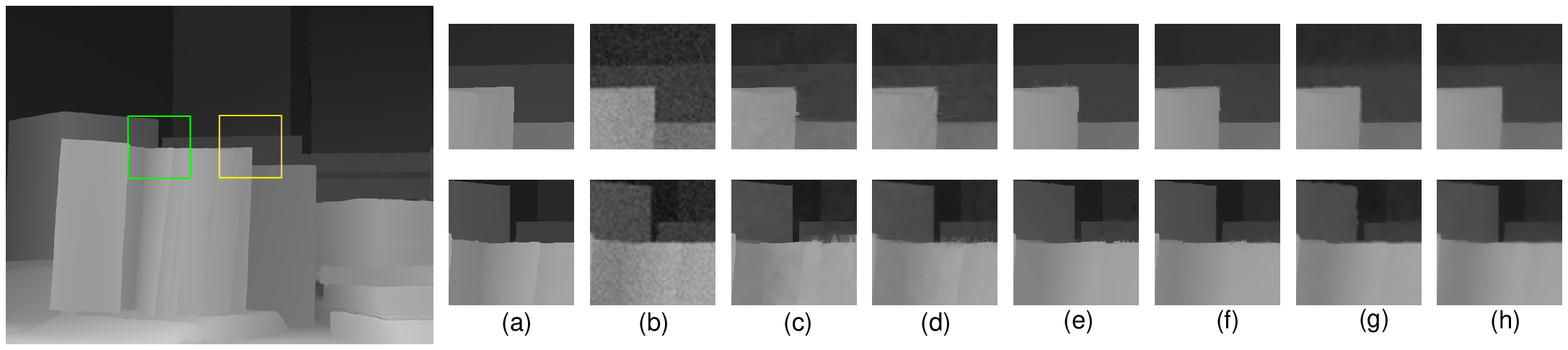}\\
  \vspace{-1cm}
  \caption{Visual comparison of $Books$ with cropped zoomed regions (scaling factor = 4).(a) Ground truth. (b) Bicubic. (c) MRF+NLM \cite{Park11}. (d) GIF\cite{He13}. (e) TGV\cite{Ferstl13}. (f) FGI\cite{LiDo16}. (g) GFL0. (h) Our proposed method}\label{figbooks}
\end{figure*}
\begin{figure*}[!t]
  \hspace{-1.5cm}
  \includegraphics[width=1.25\linewidth]{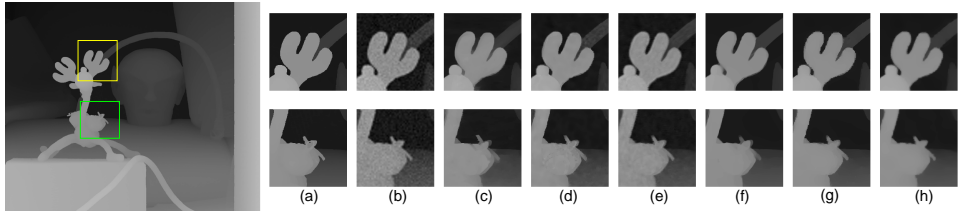}\\
  \vspace{-1cm}
  \caption{Visual comparison of $Reindeer$ with cropped zoomed regions (scaling factor = 4).(a) Ground truth. (b) Bicubic. (c) MRF+NLM \cite{Park11}. (d) CLMF\cite{Lu12}. (e) GIF\cite{He13}. (f) TGV\cite{Ferstl13}. (g) FGI\cite{LiDo16}. (h) Our proposed method}\label{figreindeer}
\end{figure*}

In our final experiment we evaluate our method on the challenging ToFMark dataset \cite{Ferstl13} consisting of three time-of-
ight (ToF) pairs, $Books$, $Shark$ and $Devil$, with ground-truth depth maps. The depth
maps are of size $120 \times 160$, and the intensity images are of size
$610 \times 810$. This corresponds to anupsampling factor of approximately $\times 5$.
In the low-resolution depth maps we add depth dependent noise and back
project the remaining points to the target camera coordinate system.

We compare our results to simple bicubic interpolation, and three state-of-the-art depth map super-resolution methods that utilize an additional guidance image as input. The quantitative results are presented in
Table 2 as RMSE in $mm$. Even on this difficult dataset we are at least on par with other four
classic or state-of-the-art methods for all the three test cases.

   \begin{table}[h]
    \caption{Results on real Time-of-Flight data from the ToFMark benchmark dataset. We report the error as
RMSE in mm and highlight the best result in boldface}
    \centering
    \setlength{\tabcolsep}{9mm}{
    \begin{tabular}{cccc}
    \hline
    & {Books} & {Devil} & {Shark}\\
    \hline
    Bicubic& 27.78  & 26.25 & 31.68  \\

    JGF\cite{Liu13} & 30.14  & 32.84 & 35.73   \\

    CLMF\cite{Lu12} & 25.67  & 23.86 & 28.93  \\

    TGV\cite{Ferstl13}& 24.68  &  23.19 &29.89  \\

    Ours & \textbf{24.53} &\textbf{23.04}& \textbf{28.46}\\
    \hline
    \end{tabular}
    }
    \end{table}\label{tabToF01}

\begin{figure*}[!t]
  \hspace{-1.5cm}
  \includegraphics[width=1.25\linewidth]{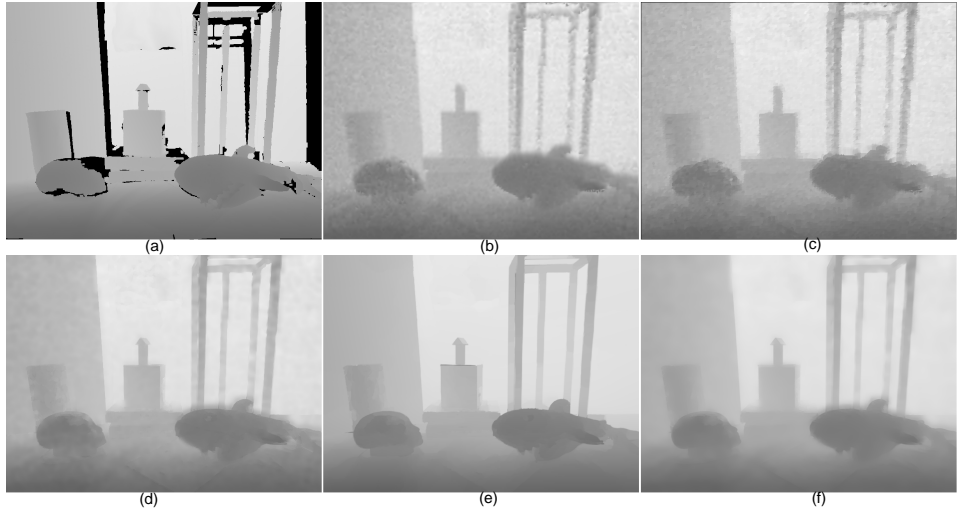}\\
  \vspace{-1.5cm}
  \caption{Qualitative results for the ToFMark dataset sample $Shark$.(a) Ground truth. (b) Bicubic. (c) JGF\cite{Liu13}. (d) CLMF\cite{Lu12}. (e)TGV\cite{Ferstl13}. (f) Our proposed method}\label{figshark}
\end{figure*}

\begin{figure*}[!t]
  \hspace{1.5cm}
  \includegraphics[width=0.75\linewidth]{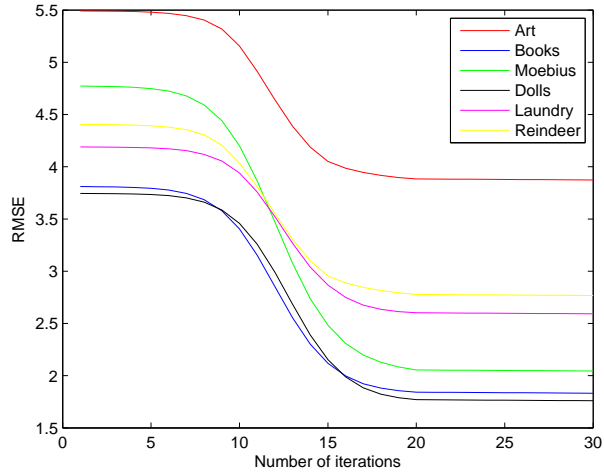}\\
  \vspace{-1cm}
  \caption{Change of the RMSE with iterations with upscaling factor 4}\label{figconvg}
\end{figure*}

In Fig.\ref{figshark}, we show the result of upsampling the depth image $Shark$.
It is observed that depth maps recovered by Bicubic, JGF\cite{Liu13}, and CLMF\cite{Lu12} still
contain considerable amount of noise, while the results obtained by TGV\cite{Ferstl13} and our method are much more clear.
By closer inspection, the TGV method in some cases introduces faulty structures in
regions where the associated intensity image has rich textures,
e.g., the shape of the shark's dorsal fin is deformation.

\subsection{Convergence}

Since the guided image filter is highly nonlinear, it is difficult to prove the global convergence of our algorithm in theory.
In this work, we only provide empirical evidence to show the stability of the proposed algorithm.

In Fig.\ref{figconvg}, we show the convergence properties of the proposed method
for test images with upscaling factor 4.  One can see that all the RMSE curves reduce monotonically with the increasing of
iteration number, and finally become stable and flat. One can also found that 30 iterations are typically sufficient.

\section{Conclusion}

This work presents a new framework to recover depth
maps from low quality measurements. Based on
the gradient statistics of depth images, we propose the low
gradient regularization and combine it with the guided filter
into the depth upsampling approach. And we present a solution to the proposed low gradient minimization problem based on threshold shrinkage. In a quantitative evaluation using widespread
datasets (Middlebury and TofMark) we show that the proposed algorithm clearly performs better than existing
state of the art methods in terms of RMSE.

\section*{References}

\bibliography{mybibfile}

\end{document}


\begin{frontmatter}

\title{Depth Image Upsampling based on Guided Filter with Low Gradient Minimization}


\author[mymainaddress]{Hang Yang\corref{mycorrespondingauthor}}
\cortext[mycorrespondingauthor]{Corresponding author}
\ead{yanghang@ciomp.ac.cn}

\author[mysecondaryaddress]{Zhongbo Zhang}

\address[mymainaddress]{Changchun Institute of Optics, Fine Mechanics and Physics, Chinese Academy of Science,
Changchun 130033, China}
\address[mysecondaryaddress]{Jilin University, Changchun, 130012,China}

\begin{abstract}
The supplementary material of our $\textbf{Signal Processing}$ submission provides the detail proofs of the Theorem 3.2 and the
relationship of three functions $1 + \sqrt{\alpha(1-t)}$, $\sqrt{\alpha}$ and $\frac{1+\alpha t}{2}$ shown in Fig.2.

\end{abstract}

\end{frontmatter}


\section{Overview}

In this supplementary material, we consider the solution of following equation:
\begin{equation}\label{sec2.e01}
    \sum_{i} \min_{x_{i}} \{(x_{i} - p_{i})^{2} + \alpha H^{t}(p_{i})\}
\end{equation}
where
\begin{equation}\label{sec2.e02}
H^{t}(p)=
   \begin{cases}
   0 &\mbox{if p = 0}\\
   t &\mbox{if $0 <\mid p \mid \leq 1$}\\
   1 &\mbox{if $\mid p \mid > 1$}
   \end{cases}
\end{equation}
Each single term w.r.t. pixel $p_{i}$ in Eq. (\ref{sec2.e02}) is
\begin{equation}\label{t.1}
    E(p) = \min_{p} (x - p)^{2} + \alpha H^{t}(p)
\end{equation}

\begin{theorem}

When $|x|\geq 1$, Eq.(\ref{t.1}) reaches its minimum $E^{*}_{p}$ under the condition
\begin{equation}\label{t.12}
   p =
   \begin{cases}
   0 &\mbox{if $ \mid x \mid \leq \min(\frac{1+\alpha t}{2}, \sqrt{\alpha})$}\\
   $$sgn(x)$$ &\mbox{if $\frac{1+\alpha t}{2} < \mid x \mid \leq 1 + \sqrt{\alpha(1-t)}$}\\
   x &\mbox{if $ \mid x \mid > \max(1 + \sqrt{\alpha(1-t)}, \sqrt{\alpha})$}
   \end{cases}
\end{equation}
\end{theorem}\label{thm02}

In our $\textbf{Signal Processing}$ submission, we have give the proofs of Theorem 1.1.

In order to make the proof more intuitive, in Fig.2, we show the relationship of three functions $1 + \sqrt{\alpha(1-t)}$, $\sqrt{\alpha}$ and $\frac{1+\alpha t}{2}$.

\begin{figure*}[!t]
\begin{minipage}[t]{0.5\linewidth}
\centering
\hspace*{-0.75cm}
\includegraphics[width=1.1\linewidth]{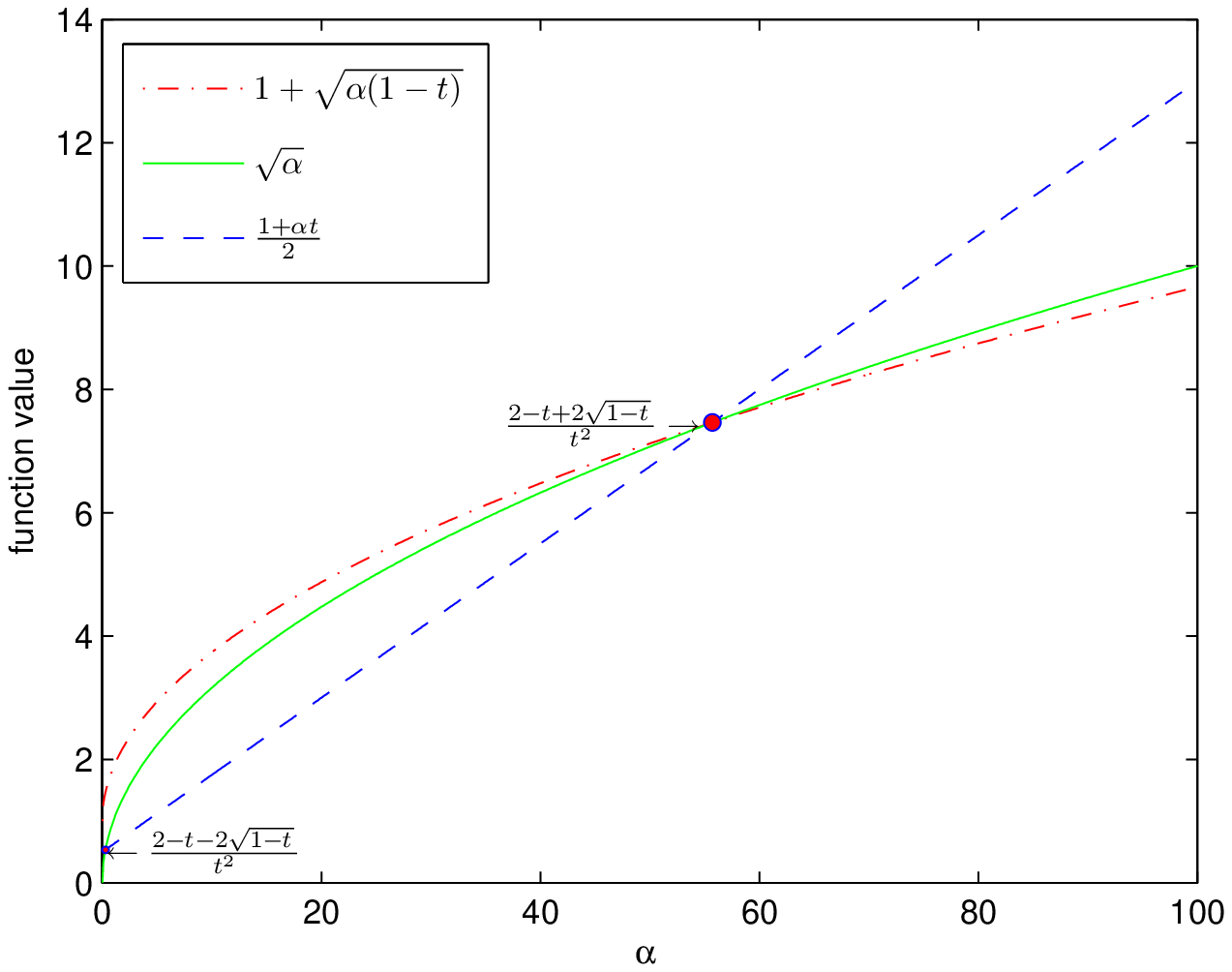}
\label{fig:side:a}
\end{minipage}%
\begin{minipage}[t]{0.5\linewidth}
\centering
\hspace*{0.0cm}
\includegraphics[width=1.1\linewidth]{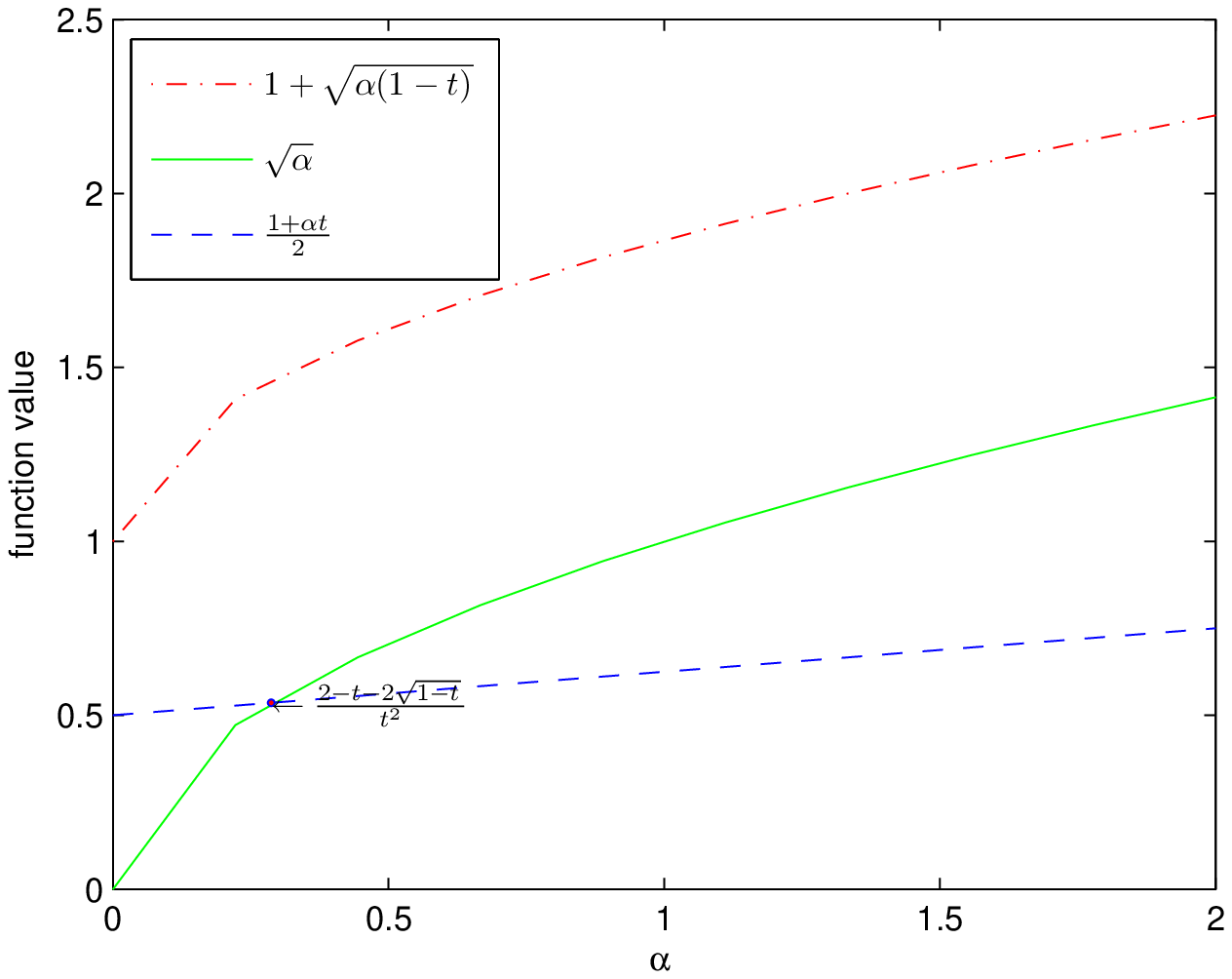}
\label{fig:side:b}
\end{minipage}
\vspace*{-0.35cm}
\caption{The relation of the three curves when $t = 0.25$. Left subplot: $\alpha \in [0,100]$. Right subplot: $\alpha \in [0,2]$.}\label{figCurve}
\end{figure*}

\begin{theorem}

When $|x| \geq 1$, and $\alpha > \frac{2-t + 2\sqrt{1-t}}{t^{2}}$, Eq.(\ref{t.1}) reaches its minimum $E^{*}_{p}$ under the condition
\begin{equation}\label{t.3.1}
   p =
   \begin{cases}
   0 &\mbox{if $ \mid x \mid \leq  \sqrt{\alpha}$}\\
   x &\mbox{if $ \mid x \mid > \sqrt{\alpha}$}
   \end{cases}
\end{equation}

When $|x|\geq 1$, and $\frac{2-t + 2\sqrt{1-t}}{t^{2}} \geq \alpha \geq \frac{2-t - 2\sqrt{1-t}}{t^{2}}$, Eq.(\ref{t.1}) reaches its minimum $E^{*}_{p}$ under the condition
\begin{equation}\label{t.3.2}
   p =
   \begin{cases}
   0 &\mbox{if $ \mid x \mid \leq \frac{1+\alpha t}{2}$}\\
   $$sgn(x)$$ &\mbox{if $\frac{1+\alpha t}{2} < \mid x \mid \leq 1 + \sqrt{\alpha(1-t)}$}\\
   x &\mbox{if $ \mid x \mid > 1 + \sqrt{\alpha(1-t)}$}
   \end{cases}
\end{equation}

When $|x| \geq 1$, and $\frac{2-t - 2\sqrt{1-t}}{t^{2}} > \alpha > 0$, Eq.(\ref{t.1}) reaches its minimum $E^{*}_{p}$ under the condition
\begin{equation}\label{t.3.3}
   p =
   \begin{cases}
   $$sgn(x)$$ &\mbox{if $1 \leq \mid x \mid \leq 1 + \sqrt{\alpha(1-t)}$}\\
   x &\mbox{if $ \mid x \mid > 1 + \sqrt{\alpha(1-t)}$}
   \end{cases}
\end{equation}
\end{theorem}

\begin{proof}

According to Lemma 1.2, we can see that the Eq.(\ref{t.3.2}) is true obviously.
We only need to proof the Eq.(\ref{t.3.1})  and (\ref{t.3.3}).

When $\frac{2-t - 2\sqrt{1-t}}{t^{2}} > \alpha > 0$, according to Lemma 1.2 and Theorem 1.1,  Eq.(\ref{t.1}) reaches its minimum $E^{*}_{p}$ under the condition
\begin{equation}\label{t.3.5}
   p =
   \begin{cases}
   0 &\mbox{if $ \mid x \mid \leq  \sqrt{\alpha}$}\\
   $$sgn(x)$$ &\mbox{if $\frac{1+\alpha t}{2} < \mid x \mid \leq 1 + \sqrt{\alpha(1-t)}$}\\
   x &\mbox{if $ \mid x \mid > \sqrt{\alpha}$}
   \end{cases}
\end{equation}
but according to Lemma 1.2,  in this case, $\frac{1+\alpha t}{2} \leq 1 + \sqrt{\alpha(1-t)}$ is false (see Eq.(\ref{t.5}).
Then, Eq.(\ref{t.3.5}) can be reducible to Eq.(\ref{t.3.1}).

According to Lemma 1.2, when $\frac{2-t - 2\sqrt{1-t}}{t^{2}} > \alpha > 0$, Eq.(\ref{t.1}) reaches its minimum $E^{*}_{p}$ under the condition
\begin{equation}\label{t.3.4}
   p =
   \begin{cases}
   0 &\mbox{if $ \mid x \mid \leq \sqrt{\alpha}$}\\
   $$sgn(x)$$ &\mbox{if $\frac{1+\alpha t}{2} < \mid x \mid \leq 1 + \sqrt{\alpha(1-t)}$}\\
   x &\mbox{if $ \mid x \mid > 1 + \sqrt{\alpha(1-t)}$}
   \end{cases}
\end{equation}
but one can see that $\sqrt{\alpha} < 1 , \frac{1+\alpha t}{2} <1$ in this case, and $\mid x \mid \geq 1$, so $\mid x \mid \leq \sqrt{\alpha}$ is false. Then, Eq.(\ref{t.3.4}) can be written as Eq.(\ref{t.3.3}).

\end{proof}

\begin{lemma}
   If ~ $0<t<1$, then $0<\frac{2-t - 2\sqrt{1-t}}{t^{2}}<1$
\end{lemma}\label{lem01}

\begin{proof}
\begin{eqnarray}\label{lp.0}
\nonumber
   -3t^{2} - t^{3} <0 \\
\nonumber
   (t+2)^{2}(1-t)<4\\
\nonumber
  (1-t)(2+t)<2\sqrt(1-t)\\
\nonumber
  2-t-2\sqrt(1-t) < t^{2}\\
  \frac{2-t - 2\sqrt{1-t}}{t^{2}}<1
\end{eqnarray}

\begin{eqnarray}\label{lp.00}
\nonumber
   t^{2} > 0 \\
\nonumber
   (2-t)^{2}>4(1-t)\\
\nonumber
  2-t-2\sqrt(1-t) >0 \\
  \frac{2-t - 2\sqrt{1-t}}{t^{2}} > 0
\end{eqnarray}
\end{proof}

\begin{lemma}
If $\alpha > \frac{2-t + 2\sqrt{1-t}}{t^{2}}$,  we have
\begin{equation}\label{t.5}
    \frac{1+ \alpha t}{2} > \sqrt{\alpha} > 1 + \sqrt{\alpha(1-t)}
\end{equation}

If $ \frac{2-t + 2\sqrt{1-t}}{t^{2}} \geq \alpha \geq \frac{2-t - 2\sqrt{1-t}}{t^{2}}$,  we have
\begin{equation}\label{t.6}
    1 + \sqrt{\alpha(1-t)} \geq \sqrt{\alpha} \geq \frac{1+ \alpha t}{2}
\end{equation}

If $ \frac{2-t - 2\sqrt{1-t}}{t^{2}} > \alpha > 0 $,  we have
\begin{equation}\label{t.7}
    1 + \sqrt{\alpha(1-t)} >  \frac{1+ \alpha t}{2} > \sqrt{\alpha}
\end{equation}
\end{lemma}\label{lem02}

\begin{proof}

 First, we give a fact,
\begin{equation}\label{t.3}
    t^{2}\alpha^{2} +(2t -4) \alpha + 1 = 0
\end{equation}
because $t<1$ and the discriminant value  of the equation is $(2t-4)^{2} - 4t^{2} = 16(1-t)>0$.
Then,this equation has two real solutions, they are
\begin{equation}\label{t.4}
    \alpha = \frac{2-t \pm 2\sqrt{1-t}}{t^{2}}
\end{equation}
where $\frac{2-t + 2\sqrt{1-t}}{t^{2}} > 1$ and according to Lemma 1.1, we have $\frac{2-t - 2\sqrt{1-t}}{t^{2}}<1$

1) When $\alpha > \frac{2-t + 2\sqrt{1-t}}{t^{2}}$, we can see that $\alpha > 1$
\begin{eqnarray}\label{lp.1}
\nonumber
   t^{2}\alpha^{2} +(2t -4) \alpha + 1 & > &0 \\
\nonumber
   (t\alpha +1)^{2} & > & 4\alpha\\
\nonumber
  t\alpha +1 & > & 2\sqrt{\alpha}\\
\nonumber
    \alpha +1 - 2\sqrt{\alpha} & > & \alpha(1-t)\\
    \sqrt{\alpha}& > & 1+\sqrt{\alpha(1-t)}
\end{eqnarray}
and
\begin{eqnarray}\label{lp.2}
\nonumber
   t^{2}\alpha^{2} +(2t -4) \alpha + 1 & > &0 \\
\nonumber
   (t\alpha +1)^{2} & > & 4\alpha\\
  \frac{1 + \alpha t}{2} & > & \sqrt{\alpha}
\end{eqnarray}

2)When $\frac{2-t + 2\sqrt{1-t}}{t^{2}} \geq \alpha \geq \frac{2-t - 2\sqrt{1-t}}{t^{2}}$
\begin{eqnarray}\label{lp.3}
\nonumber
   t^{2}\alpha^{2} +(2t -4) \alpha + 1 & \leq  &0 \\
\nonumber
  t\alpha +1 & \leq & 2\sqrt{\alpha}\\
\frac{1 + \alpha t}{2} & \leq & \sqrt{\alpha}
\end{eqnarray}
and when $\frac{2-t + 2\sqrt{1-t}}{t^{2}} \geq \alpha\geq 1$, we have
\begin{eqnarray}\label{lp.5}
\nonumber
   t^{2}\alpha^{2} +(2t -4) \alpha + 1 & \leq &0 \\
\nonumber
    (\sqrt{\alpha} -1)^{2}& \leq & \alpha(1-t)\\
    \sqrt{\alpha} & \leq & 1+ \sqrt{\alpha(1-t)}
\end{eqnarray}

When $ \frac{2-t - 2\sqrt{1-t}}{t^{2}} \leq \alpha<1$, we have
\begin{equation}\label{lp.6}
    \sqrt{\alpha} < 1 < 1+ \sqrt{\alpha(1-t)}
\end{equation}

So, in summary, when $\frac{2-t + 2\sqrt{1-t}}{t^{2}} \geq \alpha \geq \frac{2-t - 2\sqrt{1-t}}{t^{2}}$, we have Eq.(\ref{lp.6}).

3) When $ \frac{2-t - 2\sqrt{1-t}}{t^{2}} > \alpha > 0 $,
\begin{equation}\label{lp.8}
    t^{2}\alpha^{2} +(2t -4) \alpha + 1 > 0
\end{equation}
similar to the Eq.(\ref{lp.2}), we can obtain $\frac{1 + \alpha t}{2} > \sqrt{\alpha}$.

And according to Lemma 1.1, $\alpha< \frac{2-t - 2\sqrt{1-t}}{t^{2}} <1$,
\begin{equation}\label{lp.7}
     \frac{1 + \alpha t}{2}  < 1 < 1+ \sqrt{\alpha(1-t)}
\end{equation}
\end{proof}